\documentclass[10pt,singlecolumn]{article}

\usepackage[margin=0.8in]{geometry}
\usepackage{amsmath,amssymb,amsfonts,amsthm}
\usepackage{booktabs}
\usepackage{graphicx}
\usepackage{hyperref}
\usepackage{xcolor}
\usepackage{algorithm}
\usepackage{algorithmic}
\usepackage{caption}
\usepackage{subcaption}
\usepackage[numbers]{natbib}
\usepackage{microtype}
\usepackage{enumitem}
\usepackage{multirow}
\hypersetup{colorlinks=true,linkcolor=blue,citecolor=blue,urlcolor=blue}

\newtheorem{proposition}{Proposition}

\title{\textbf{IsoQuant: Hardware-Aligned SO(4) Isoclinic Rotations \\ for LLM KV Cache Compression}}

\author{
  Zhongping Ji \\
  \texttt{}
}

\date{}

\begin{document}
\maketitle

\begin{abstract}

Orthogonal feature decorrelation is a useful primitive for low-bit online vector quantization, but the standard approach of applying a dense random orthogonal matrix incurs prohibitive $O(d^2)$ storage and compute. Recent work such as RotorQuant reduces this cost by replacing the global transform with blockwise $3$D Clifford rotors. While effective, the resulting $3$D partition is not naturally aligned with modern hardware, since transformer feature widths are typically powers of two and thus lead to awkward tails, irregular memory access, and limited local mixing.

We propose \textbf{IsoQuant}, a blockwise rotation framework based on quaternion algebra and the isoclinic decomposition of $SO(4)$. Each block of four coordinates is identified with a quaternion and transformed by the closed-form map $T(v)=q_L v \overline{q_R}$ parameterized by a pair of unit quaternions. This yields two main variants: \emph{IsoQuant-Full}, which realizes the full six degrees of freedom of $SO(4)$, and \emph{IsoQuant-Fast}, which retains a single isoclinic factor for lower overhead. The same framework also admits a lightweight $2$D planar special case, which we treat as an auxiliary operating point rather than the primary method. At $d=128$, IsoQuant-Full reduces forward rotation cost from approximately $2{,}408$ FMAs in RotorQuant to $1{,}024$, while IsoQuant-Fast reduces it further to $512$. Across $18$ fused CUDA benchmark settings spanning $d \in \{128,256,512\}$, bit widths $\{2,3,4\}$, and both FP16 and FP32 execution, IsoQuant-Full, IsoQuant-Fast, and the $2$D special case achieve mean kernel-level speedups of $4.49\times$, $4.66\times$, and $4.66\times$ over RotorQuant while maintaining comparable reconstruction MSE, with peak speedups above $6\times$. Our current validation is limited to the stage-1 quantize--dequantize path on synthetic normalized vectors; full end-to-end KV-cache evaluation on real model activations remains future work. Code: \url{https://github.com/ParaMind2025/isoquant}.

\end{abstract}

\section{Introduction}
\label{sec:intro}

KV-cache compression has emerged as a central systems bottleneck for long-context large language model inference. A core insight behind online vector quantization methods such as TurboQuant~\cite{turboquant} is that decorrelating features before scalar quantization substantially improves rate--distortion behavior. In particular, a random orthogonal transform spreads information across coordinates, making per-coordinate Lloyd--Max quantization far more effective than quantization in the original basis.

The main drawback is cost. For a head dimension $d$, a dense orthogonal transform requires $O(d^2)$ parameters and arithmetic, which is difficult to justify in latency-sensitive settings such as autoregressive decoding. RotorQuant~\cite{rotorquant} addresses this issue by replacing the dense transform with sparse blockwise rotations in the geometric algebra $\mathrm{Cl}(3,0)$. This preserves the qualitative benefit of local decorrelation while reducing complexity to linear in $d$.

Despite this progress, the $3$D block structure leaves performance on the table. First, it is not hardware aligned: common head dimensions such as $64$, $128$, and $256$ are powers of two, so partitioning into triples creates awkward remainder handling and non-ideal memory layout. For example, $d=128$ yields forty-two full $3$D blocks plus one $2$D tail. Second, a $3$D rotation has only three intrinsic degrees of freedom, which limits how aggressively each local subspace can mix correlated coordinates.

This paper studies a different operating point. We move from $3$D Clifford blocks to $4$D quaternion blocks and leverage the Lie-theoretic decomposition
\begin{equation}
so(4) \cong su(2) \oplus su(2),
\end{equation}
which implies that every $4$D rotation can be represented by a pair of unit quaternions acting from the left and right. This yields a closed-form, low-overhead parameterization of $SO(4)$ that is both mathematically clean and implementation-friendly. We therefore name the method \textbf{IsoQuant}, emphasizing the isoclinic structure underlying its blockwise $SO(4)$ transform.

\paragraph{Contributions.}
Our contributions are as follows.
\begin{enumerate}[nosep,leftmargin=14pt]
  \item We introduce \textbf{IsoQuant}, a $4$D block rotation scheme for online vector quantization based on quaternion sandwich products.
  \item We derive two practical $4$D instantiations, \textbf{IsoQuant-Full} and \textbf{IsoQuant-Fast}, and show that the same framework also admits a lightweight $2$D planar special case as an additional operating point.
  \item We analyze why $4$D blocks are attractive for systems deployment: they improve arithmetic efficiency, eliminate most boundary handling, and align naturally with SIMD and fused-kernel execution.
  \item We provide a fused CUDA implementation and show, in fair kernel-level comparisons against RotorQuant fused CUDA kernels, that IsoQuant achieves consistent speedups while preserving reconstruction quality.
\end{enumerate}

\section{Related Work}
\label{sec:related}

\paragraph{Online vector quantization and KV-cache compression.}
TurboQuant~\cite{turboquant} frames online vector quantization as a dense-rotation-plus-scalar-quantization pipeline and combines it with QJL residual correction~\cite{qjl}. RotorQuant~\cite{rotorquant} reduces the heavy dense transform by replacing it with sparse $3$D Clifford rotors. Other KV-cache compression methods such as KIVI~\cite{kivi} and KVQuant~\cite{kvquant} focus on calibration, asymmetric quantization, or hardware-aware packing rather than explicit geometric decorrelation.

\paragraph{Quaternion and geometric representations.}
Unit quaternions are a classical tool for representing rotations~\cite{kuipers}. In machine learning, algebraically structured representations have appeared in equivariant models, geometric neural networks, and efficient transformations on structured features. RiemannFormer~\cite{riemannformer} provides a recent geometric treatment that explicitly discusses the isoclinic decomposition of $SO(4)$ in the setting of curved-space attention. Our use of this structure is narrower and more systems-oriented: we exploit the same $SO(4)$ factorization to design a cheap blockwise decorrelating transform for low-bit quantization.

\section{Problem Setup}
\label{sec:setup}

We consider the stage-1 mean-squared-error component of online vector quantization. Given an input vector $x \in \mathbb{R}^d$, we seek an encoder $E$, scalar quantizer $Q$, and decoder $D$ such that
\begin{equation}
\hat{x} = D(Q(E(x)))
\end{equation}
minimizes reconstruction error while keeping both parameter count and online compute small.

As in prior work~\cite{turboquant,rotorquant}, we factor the problem into three pieces:
\begin{enumerate}[nosep,leftmargin=14pt]
  \item an orthogonal or approximately orthogonal transform that decorrelates coordinates;
  \item coordinate-wise scalar quantization in the transformed basis;
  \item an inverse transform to reconstruct the vector.
\end{enumerate}

To stabilize the scalar quantizer, we additionally separate norm and direction. Writing
\begin{equation}
x = \rho \, \bar{x}, \qquad \rho = \|x\|_2, \qquad \|\bar{x}\|_2 = 1,
\end{equation}
we quantize the normalized direction $\bar{x}$ and store or transmit the norm $\rho$ separately. This follows the implementation pattern used by efficient quantizers and keeps the transformed coordinates within a predictable dynamic range.

\section{Quaternion and $SO(4)$ Preliminaries}
\label{sec:prelim}

We begin by identifying each $4$D block with a quaternion
\begin{equation}
v = x_0 + x_1 \mathbf{i} + x_2 \mathbf{j} + x_3 \mathbf{k} \in \mathbb{H},
\end{equation}
where $\mathbb{H}$ denotes the quaternion algebra and $\mathbf{i}^2 = \mathbf{j}^2 = \mathbf{k}^2 = \mathbf{ijk} = -1$~\cite{kuipers}. For
\begin{equation}
q = a + b\mathbf{i} + c\mathbf{j} + d\mathbf{k},
\end{equation}
its quaternion conjugate is
\begin{equation}
\overline{q} = a - b\mathbf{i} - c\mathbf{j} - d\mathbf{k}.
\end{equation}
Unit quaternions form the $3$-sphere $S^3$ and provide a nonsingular parameterization of rotational factors that will be central to our construction.

The relevant geometric structure is the special orthogonal group in four dimensions. At the Lie-algebra level one has the classical decomposition
\begin{equation}
\mathfrak{so}(4) \cong \mathfrak{su}(2)_L \oplus \mathfrak{su}(2)_R,
\end{equation}
which underlies the isoclinic decomposition of $SO(4)$ discussed, for example, in RiemannFormer~\cite{riemannformer}. This splitting implies that every infinitesimal $4$D rotation can be written as the sum of two commuting components. More precisely, if
\begin{equation}
X = X_A + X_B, \qquad X_A \in \mathfrak{su}(2)_L,\quad X_B \in \mathfrak{su}(2)_R,
\end{equation}
then $[X_A,X_B]=0$, and therefore the corresponding group element factors as
\begin{equation}
R = \exp(X) = \exp(X_A + X_B) = \exp(X_A)\exp(X_B),
\end{equation}
where the exponential map is given by
\begin{equation}
\exp: \mathfrak{so}(4) \to SO(4), \qquad X \mapsto \exp(X) = \sum_{k=0}^{\infty} \frac{X^k}{k!}.
\end{equation}
This commuting factorization is the Lie-theoretic counterpart of the left/right isoclinic decomposition.

\begin{proposition}
Let $q_L, q_R \in S^3$ be unit quaternions. Then the map
\begin{equation}
T(v) = q_L \, v \, \overline{q_R}
\end{equation}
defines an orthogonal transformation of $\mathbb{R}^4$. Its inverse is
\begin{equation}
T^{-1}(v) = \overline{q_L} \, v \, q_R,
\end{equation}
and the pairs $(q_L,q_R)$ and $(-q_L,-q_R)$ induce the same element of $SO(4)$.
\end{proposition}

\begin{proof}[Proof sketch]
Quaternion multiplication by a unit quaternion preserves the Euclidean norm on $\mathbb{H} \cong \mathbb{R}^4$, both for left multiplication and for right multiplication by the conjugate. Hence the composition $v \mapsto q_L v \overline{q_R}$ is norm-preserving and therefore orthogonal. The inverse follows immediately from associativity of quaternion multiplication together with $q_L \overline{q_L} = \overline{q_R} q_R = 1$. Finally, replacing both $q_L$ and $q_R$ by their negatives leaves the map unchanged because
\begin{equation}
(-q_L)\, v \, \overline{(-q_R)} = (-q_L)\, v \, (-\overline{q_R}) = q_L v \overline{q_R}.
\end{equation}
\end{proof}

\noindent
The left and right quaternion factors correspond to the left-isoclinic and right-isoclinic components, respectively. Hence a pair of unit quaternions provides a compact, closed-form parameterization of a general $4$D rotation, up to the standard double-cover ambiguity. This parameterization is especially attractive for quantization because it realizes the full six degrees of freedom of $SO(4)$ without ever requiring explicit storage of a dense $4 \times 4$ orthogonal matrix.

\subsection{Block-Diagonal Rotations in $SO(4g)$}
\label{sec:block_so4g}

Let
\begin{equation}
g = \left\lceil \frac{d}{4} \right\rceil .
\end{equation}
After zero padding when necessary, the input is naturally viewed as an element of $\mathbb{R}^{4g}$. In a fully dense formulation one would apply an unrestricted orthogonal transform in $SO(4g)$, which plays the role of a padded surrogate for the original $d \times d$ orthogonal rotation. IsoQuant instead restricts this transform to a structured Lie subgroup obtained from independent $4$D rotational blocks.

Formally, define the Lie subalgebra
\begin{equation}
\mathfrak{g} = \bigoplus_{i=1}^{g} \mathfrak{so}(4)^{(i)} \subset \mathfrak{so}(4g).
\end{equation}
Any element $H \in \mathfrak{g}$ has block-diagonal form
\begin{equation}
H = \mathrm{diag}(X_1, X_2, \dots, X_g),
\end{equation}
where each $X_i \in \mathfrak{so}(4)$ and all off-diagonal blocks vanish. Since the exponential map preserves block-diagonal structure, one obtains
\begin{equation}
\exp(H) = \mathrm{diag}(\exp(X_1), \exp(X_2), \dots, \exp(X_g)).
\end{equation}
Because $\exp(X_i) \in SO(4)$ for each block, the resulting matrix $\exp(H)$ belongs to $SO(4g)$. The set of all such matrices forms an embedded Lie subgroup isomorphic to
\begin{equation}
SO(4) \times SO(4) \times \cdots \times SO(4),
\end{equation}
with $g$ factors. In this sense, IsoQuant replaces a dense high-dimensional orthogonal transform by a block-diagonal subgroup of $SO(4g)$ composed of local $SO(4)$ actions.

The matrix-exponential viewpoint is primarily conceptual: it makes clear which subgroup of $SO(4g)$ is being used and how this subgroup arises from a direct-sum Lie algebra. In the actual computation, however, we do not construct the skew-symmetric blocks $X_i$ or evaluate the corresponding matrix exponentials. Each local $SO(4)$ action is instead realized more efficiently by a quaternion pair $(q_L^{(i)}, q_R^{(i)})$, which yields the same block rotation in closed form while avoiding dense matrix materialization.

\section{Method: IsoQuant}
\label{sec:method}

\subsection{Blockwise 4D Rotation}

We now instantiate the subgroup construction above as a stage-1 quantization transform. Let $\bar{x} \in \mathbb{R}^d$ denote the normalized input direction. We partition $\bar{x}$ into
\begin{equation}
g = \left\lceil \frac{d}{4} \right\rceil
\end{equation}
blocks,
\begin{equation}
\bar{x} = \left[v^{(1)}, v^{(2)}, \dots, v^{(g)}\right],
\end{equation}
where each $v^{(i)} \in \mathbb{R}^4$ is identified with an element of $\mathbb{H}$. If $d$ is not divisible by $4$, the final block is zero padded, so the full transformed object lies in $\mathbb{R}^{4g}$.

From the perspective of Section~\ref{sec:block_so4g}, the encoder acts on the padded vector by an element of the block-diagonal subgroup $(SO(4))^g \subset SO(4g)$. The algorithm does not materialize this action as a dense matrix. Instead, each local $SO(4)$ factor is represented by one or two unit quaternions and applied directly in closed form.

Accordingly, each block undergoes the sequence
\begin{equation}
v^{(i)} \mapsto \tilde{v}^{(i)} \mapsto \hat{v}^{(i)} \mapsto v_{\mathrm{rec}}^{(i)}.
\end{equation}
That is, one first applies a local $SO(4)$ rotation, then performs coordinate-wise scalar quantization in the rotated basis, and finally applies the inverse local rotation. The recovered blocks are concatenated and the original norm $\rho$ is restored.

\subsection{IsoQuant-Full}

IsoQuant-Full uses the complete double-sided action of $SO(4)$. For block $i$, we maintain a pair of unit quaternions $(q_L^{(i)}, q_R^{(i)})$ and compute
\begin{align}
\tilde{v}^{(i)} &= q_L^{(i)} v^{(i)} \overline{q_R^{(i)}}, \\
\hat{v}^{(i)} &= Q\!\left(\tilde{v}^{(i)}\right), \\
v_{\mathrm{rec}}^{(i)} &= \overline{q_L^{(i)}} \hat{v}^{(i)} q_R^{(i)}.
\end{align}
This uses the full six-dimensional rotational freedom of $SO(4)$ and offers the strongest local mixing.

\subsection{IsoQuant-Fast}

IsoQuant-Fast restricts the transform to a single isoclinic factor:
\begin{align}
\tilde{v}^{(i)} &= q_L^{(i)} v^{(i)}, \\
\hat{v}^{(i)} &= Q\!\left(\tilde{v}^{(i)}\right), \\
v_{\mathrm{rec}}^{(i)} &= \overline{q_L^{(i)}} \hat{v}^{(i)}.
\end{align}
Geometrically, this corresponds to a $3$-dimensional subgroup of $SO(4)$ isomorphic to $SO(3)$. It sacrifices expressivity for lower parameter count, lower arithmetic cost, and simpler kernels.

\subsection{A Lightweight $2$D Special Case}

Although IsoQuant is fundamentally motivated by blockwise $SO(4)$ structure, the same implementation philosophy admits a degenerate planar special case on coordinate pairs. We partition the normalized input into $2$D blocks
\begin{equation}
u^{(j)} \in \mathbb{R}^2,
\end{equation}
and apply a standard planar rotation
\begin{align}
\tilde{u}^{(j)} &= R(\theta^{(j)}) u^{(j)}, \\
\hat{u}^{(j)} &= Q\!\left(\tilde{u}^{(j)}\right), \\
u_{\mathrm{rec}}^{(j)} &= R(-\theta^{(j)}) \hat{u}^{(j)},
\end{align}
where
\begin{equation}
R(\theta)=
\begin{bmatrix}
\cos\theta & -\sin\theta \\
\sin\theta & \cos\theta
\end{bmatrix}.
\end{equation}
We do not present this $2$D case as a separate method family; rather, we treat it as an extremely lightweight IsoQuant operating point. Its local mixing capacity is weaker than that of the $4$D formulations, so IsoQuant-Full remains the primary and most expressive variant. Nevertheless, the $2$D special case is useful experimentally because it clarifies how much accuracy is lost, if any, when blockwise rotation cost is pushed even lower.

\subsection{Parameterization and Learning}

To avoid constrained optimization on the manifold directly, we parameterize each unit quaternion by an unconstrained vector and normalize it on the fly:
\begin{equation}
q_L^{(i)} = \frac{u_L^{(i)}}{\|u_L^{(i)}\|_2}, \qquad
q_R^{(i)} = \frac{u_R^{(i)}}{\|u_R^{(i)}\|_2},
\end{equation}
where $u_L^{(i)}, u_R^{(i)} \in \mathbb{R}^4$ are free parameters. This keeps optimization in Euclidean space while enforcing the unit-quaternion constraint implicitly in the computation graph. A practical lightweight variant is to sample the initial $u$ vectors from a Gaussian distribution and keep them fixed, yielding random block rotations analogous to the randomized transform used by TurboQuant.
When randomized rotations are desired, we sample from the Haar distribution on the corresponding rotation group: this reduces to uniform angle sampling for the $2$D special case and to Gaussian-normalize sampling on $S^3$ for the quaternion factors used in the $4$D variants.

\subsection{Quantization Pipeline}

Algorithm~\ref{alg:isoquant} summarizes the stage-1 pipeline.

\begin{algorithm}[H]
\caption{IsoQuant Stage-1 Quantization}
\label{alg:isoquant}
\begin{algorithmic}[1]
\REQUIRE Input vector $x \in \mathbb{R}^d$, scalar quantizer $Q$, mode $\in \{\textsc{Full}, \textsc{Fast}, \textsc{2D}\}$
\STATE Compute norm $\rho \leftarrow \|x\|_2$ and normalized vector $\bar{x} \leftarrow x / \max(\rho,\varepsilon)$
\STATE Partition $\bar{x}$ into zero-padded local blocks according to the mode ($4$D for \textsc{Full}/\textsc{Fast}, $2$D for \textsc{2D})
\FOR{$i=1$ to $g$}
  \IF{mode = \textsc{Full}}
    \STATE $\tilde{v}^{(i)} \leftarrow q_L^{(i)} v^{(i)} \overline{q_R^{(i)}}$
    \STATE $\hat{v}^{(i)} \leftarrow Q(\tilde{v}^{(i)})$
    \STATE $v_{\mathrm{rec}}^{(i)} \leftarrow \overline{q_L^{(i)}} \hat{v}^{(i)} q_R^{(i)}$
  \ELSIF{mode = \textsc{Fast}}
    \STATE $\tilde{v}^{(i)} \leftarrow q_L^{(i)} v^{(i)}$
    \STATE $\hat{v}^{(i)} \leftarrow Q(\tilde{v}^{(i)})$
    \STATE $v_{\mathrm{rec}}^{(i)} \leftarrow \overline{q_L^{(i)}} \hat{v}^{(i)}$
  \ELSE
    \STATE $\tilde{u}^{(i)} \leftarrow R(\theta^{(i)}) u^{(i)}$
    \STATE $\hat{u}^{(i)} \leftarrow Q(\tilde{u}^{(i)})$
    \STATE $u_{\mathrm{rec}}^{(i)} \leftarrow R(-\theta^{(i)}) \hat{u}^{(i)}$
  \ENDIF
\ENDFOR
\STATE Concatenate reconstructed blocks and drop padded coordinates
\RETURN $\hat{x} \leftarrow \rho$ times the concatenated reconstructed blocks
\end{algorithmic}
\end{algorithm}

\subsection{Probabilistic Intuition for Random Subspace Rotations}
\label{sec:prob-intuition}

A natural question is why low-dimensional random block rotations can still improve quantization in a high-dimensional vector. The key point is that quantization does not require full global mixing; it often suffices to isotropize energy locally within each block.

Let $x^{(b)} \in \mathbb{R}^k$ denote a fixed block with radius $r_b=\|x^{(b)}\|_2$, and let $R_b \sim \mathrm{Haar}(SO(k))$ be a Haar-distributed random rotation. Then
\begin{equation}
y^{(b)} = R_b x^{(b)}
\end{equation}
is distributed as a point with fixed norm $r_b$ and uniformly random direction on the sphere $S^{k-1}$. Hence, for any coordinate $j$,
\begin{equation}
\mathbb{E}[y^{(b)}_j \mid x^{(b)}] = 0,
\qquad
\mathbb{E}[(y^{(b)}_j)^2 \mid x^{(b)}] = \frac{r_b^2}{k}.
\end{equation}
Thus, random block rotation redistributes the energy of a fixed block evenly across coordinates in expectation.

More can be said about the marginal law of each rotated coordinate. If $u \sim \mathrm{Unif}(S^{k-1})$, then $y^{(b)} = r_b u$, and the marginal density of one normalized coordinate $z=u_j$ is
\begin{equation}
f_k(z) \propto (1-z^2)^{\frac{k-3}{2}}, \qquad |z| \le 1.
\end{equation}
This highlights an important difference between $k=2$ and $k=4$. For $k=2$, the marginal follows an arcsine law,
\begin{equation}
f_2(z)=\frac{1}{\pi\sqrt{1-z^2}},
\end{equation}
which places relatively more mass near the extremes. For $k=4$, the marginal becomes
\begin{equation}
f_4(z)=\frac{2}{\pi}\sqrt{1-z^2},
\end{equation}
which is more concentrated near the center and vanishes at the boundaries. Consequently, the $4$D case is structurally more favorable for scalar quantization, since individual coordinates are less likely to attain extreme values.

A complementary covariance-level view leads to the same intuition. Let $x \in \mathbb{R}^d$ have covariance $\Sigma$, partitioned into $k \times k$ blocks, and let
\begin{equation}
R = \mathrm{diag}(R_1,\dots,R_m),
\qquad
R_i \sim \mathrm{Haar}(SO(k))
\end{equation}
be an independent block-diagonal random rotation. Then
\begin{equation}
\mathbb{E}_R[R\Sigma R^\top]
=
\mathrm{diag}\!\left(
\frac{\mathrm{tr}(\Sigma_{11})}{k}I_k,\,
\dots,\,
\frac{\mathrm{tr}(\Sigma_{mm})}{k}I_k
\right).
\end{equation}
Therefore, each block becomes isotropic in expectation, while cross-block correlations vanish in expectation under independent random rotations. For any fixed realization of $R$, cross-block correlation energy is not literally removed; rather, its directions are scrambled, which mitigates worst-case structured dependencies seen by coordinate-wise quantization.

\section{Complexity Analysis}
\label{sec:complexity}

A quaternion multiplication uses 16 scalar multiplications and 12 scalar additions. To match common systems reporting conventions, we count this as approximately 16 fused multiply-add operations. For $d=128$, the forward rotation cost becomes straightforward to compare.

\begin{table}[t]
\centering
\caption{Forward rotation complexity at $d=128$.}
\label{tab:complexity}
\begin{tabular}{@{}lccc@{}}
\toprule
Method & Block Structure & Params & FMAs \\
\midrule
TurboQuant~\cite{turboquant} & dense $128 \times 128$ & $16{,}384$ & $16{,}384$ \\
RotorQuant~\cite{rotorquant} & $43 \times 3$D blocks & $172$ & $\approx 2{,}408$ \\
IsoQuant-$2$D & $64 \times 2$D blocks & $128$ & $\approx 256$ \\
IsoQuant-Full & $32 \times 4$D blocks & $256$ & $1{,}024$ \\
IsoQuant-Fast & $32 \times 4$D blocks & $128$ & $512$ \\
\bottomrule
\end{tabular}
\end{table}

Table~\ref{tab:complexity} summarizes the forward stage-1 rotation cost at $d=128$. The comparison highlights three regimes. IsoQuant-Full uses somewhat more parameters than RotorQuant because each block stores two quaternions, but it still remains tiny relative to dense rotations and cuts rotation arithmetic by more than $2\times$. IsoQuant-Fast goes further and is cheaper than RotorQuant in both parameters and compute. The $2$D special case is cheaper still, but also has the weakest local mixing and should therefore be viewed as a low-cost operating point rather than the main target of the paper.

More generally, if $g_4=\lceil d/4 \rceil$ and $g_2=\lceil d/2 \rceil$, then the parameter count is $8g_4$ for Full, $4g_4$ for Fast, and $2g_2$ for the $2$D special case. The forward rotation cost is $32g_4$ FMAs for Full, $16g_4$ FMAs for Fast, and approximately $4g_2$ FMAs for the $2$D case. All variants therefore scale linearly in $d$.

\section{Systems Considerations}
\label{sec:systems}

\subsection{Why 4D Blocks Are Hardware Friendly}

The choice of block size is not purely algebraic. It directly shapes memory layout, vectorization efficiency, and kernel fusion opportunities.

\paragraph{Alignment.}
Most transformer head dimensions are powers of two. A $4$D partition therefore avoids the pathological tails induced by $3$D chunking in almost every common setting. At $d=128$, IsoQuant uses exactly $32$ blocks with no remainder, whereas a $3$D design requires $42$ full blocks plus a leftover fragment.

\paragraph{Vectorization.}
Four-wide blocks fit naturally into SIMD-friendly load and store patterns such as \texttt{float4}. This reduces boundary checks and helps both CPU SIMD backends and GPU kernels maintain regular control flow.

\paragraph{Kernel fusion.}
The per-block transform is a closed-form sequence of one or two quaternion products, coordinate-wise scalar quantization, and an inverse transform. This structure is especially suitable for fused kernels in online quantization pipelines because the entire block can often remain in registers from input load through output store.

\subsection{Prototype Mapping}

Our current prototype follows exactly this design pattern: it packs feature vectors into $4$D blocks, applies either double-sided or single-sided quaternion rotation, performs scalar Lloyd-Max quantization, and reconstructs the result with a fused CUDA kernel. In addition, we include a lightweight $2$D planar special case that serves as an even cheaper operating point within the same implementation family. We built a benchmark harness that directly compares these fused IsoQuant kernels against the fused RotorQuant CUDA kernel under matched tensor shapes, bit widths, and data types.

\section{Compatibility with Residual Correction}
\label{sec:qjl}

IsoQuant is intended as a replacement for the stage-1 decorrelation transform, not as a rejection of later residual correction. In two-stage pipelines such as TurboQuant~\cite{turboquant}, one can keep the residual inner-product correction unchanged:
\begin{equation}
r = x - \hat{x}_{\mathrm{mse}}.
\end{equation}
The residual can still be projected with a quantized Johnson-Lindenstrauss transform~\cite{qjl} or a related low-bit correction mechanism. In this sense, IsoQuant is complementary to existing stage-2 estimators: it reduces the cost of the orthogonalization step while remaining compatible with unbiased inner-product correction.

\section{Experimental Results}
\label{sec:experiments}

\subsection{Evaluation Protocol}

This section evaluates the stage-1 quantize--dequantize path of IsoQuant. The purpose is to isolate the contribution of the decorrelation transform itself before introducing stage-2 residual correction or downstream task effects. We therefore focus on two directly attributable quantities: reconstruction MSE on normalized synthetic vectors and fused CUDA kernel latency for the blockwise stage-1 transform.

All experiments in this section use synthetic normalized vectors with batch size $8192$. We evaluate all combinations of
\begin{equation}
d \in \{128,256,512\}, \qquad b \in \{2,3,4\}, \qquad \text{dtype} \in \{\texttt{fp16}, \texttt{fp32}\},
\end{equation}
for a total of $18$ fused-kernel benchmark settings. We focus on these dimensions because they cover common per-head or grouped-KV widths in practical LLM deployments, while avoiding the large-dimension kernel-specialization regime that would otherwise dominate the discussion. On the IsoQuant side, we evaluate IsoQuant-Full, IsoQuant-Fast, and the lightweight IsoQuant-$2$D special case. On the baseline side, we compare against RotorQuant~\cite{rotorquant}, including its fused CUDA kernel in \texttt{rotor\_fused\_kernel.cu}. TurboQuant~\cite{turboquant} is used as a conceptual dense-rotation reference in the complexity analysis, not as a runtime baseline.

All fused CUDA benchmarks were conducted on a single NVIDIA RTX 4090 GPU. For every configuration, RotorQuant and IsoQuant are benchmarked under the same tensor shape, bit width, and execution dtype. In addition to kernel latency, we report reconstruction MSE, parameter count, and estimated forward arithmetic cost.

\subsection{Main Results}

Table~\ref{tab:fused-main} reports the full fused CUDA sweep. The main takeaway is that the entire IsoQuant family is consistently faster than RotorQuant under an apples-to-apples kernel comparison while preserving essentially identical reconstruction quality. Across all $18$ settings, IsoQuant-Full achieves an average speedup of $4.49\times$ over RotorQuant fused CUDA, while IsoQuant-Fast and the lightweight IsoQuant-$2$D special case achieve nearly identical averages of $4.66\times$ and $4.66\times$, respectively. The strongest settings occur in low-bit and medium-width regimes, where speedups exceed $6\times$ while MSE remains unchanged up to the reported precision.

\begin{table*}[t]
\centering
\small
\setlength{\tabcolsep}{4.5pt}
\caption{Fused CUDA comparisons against RotorQuant on normalized synthetic vectors with batch size $8192$. Latencies are in $\mu$s.}
\label{tab:fused-main}
\begin{tabular}{@{}cccccccccc@{}}
\toprule
dtype & bits & dim & RotorQuant & \shortstack{IsoQuant\\-Full} & \shortstack{IsoQuant\\-Fast} & \shortstack{IsoQuant\\-$2$D} & \shortstack{Full\\speedup} & \shortstack{Fast\\speedup} & \shortstack{$2$D\\speedup} \\
\midrule
\texttt{fp16} & 2 & 128 & 32.7 & 8.5 & 8.2 & 8.5 & 3.86 & 3.98 & 3.85 \\
\texttt{fp16} & 3 & 128 & 36.4 & 6.2 & 6.1 & 6.2 & 5.92 & 6.00 & 5.92 \\
\texttt{fp16} & 4 & 128 & 44.2 & 9.6 & 9.4 & 9.4 & 4.62 & 4.71 & 4.70 \\
\texttt{fp16} & 2 & 256 & 32.6 & 8.8 & 8.3 & 8.4 & 3.72 & 3.92 & 3.89 \\
\texttt{fp16} & 3 & 256 & 36.8 & 9.7 & 9.5 & 9.2 & 3.80 & 3.88 & 3.99 \\
\texttt{fp16} & 4 & 256 & 46.7 & 8.1 & 7.5 & 7.5 & 5.76 & 6.24 & 6.20 \\
\texttt{fp16} & 2 & 512 & 36.0 & 8.4 & 8.5 & 8.3 & 4.27 & 4.23 & 4.31 \\
\texttt{fp16} & 3 & 512 & 39.9 & 9.9 & 8.4 & 9.6 & 4.05 & 4.73 & 4.15 \\
\texttt{fp16} & 4 & 512 & 50.4 & 13.9 & 12.5 & 13.4 & 3.63 & 4.02 & 3.77 \\
\midrule
\texttt{fp32} & 2 & 128 & 33.4 & 7.3 & 7.1 & 7.0 & 4.60 & 4.68 & 4.77 \\
\texttt{fp32} & 3 & 128 & 35.1 & 6.9 & 7.6 & 7.1 & 5.07 & 4.59 & 4.91 \\
\texttt{fp32} & 4 & 128 & 44.7 & 8.2 & 8.2 & 8.3 & 5.49 & 5.48 & 5.40 \\
\texttt{fp32} & 2 & 256 & 33.8 & 7.4 & 7.2 & 7.2 & 4.59 & 4.66 & 4.66 \\
\texttt{fp32} & 3 & 256 & 37.9 & 7.2 & 7.0 & 6.9 & 5.29 & 5.45 & 5.47 \\
\texttt{fp32} & 4 & 256 & 47.9 & 8.1 & 7.6 & 7.5 & 5.92 & 6.31 & 6.39 \\
\texttt{fp32} & 2 & 512 & 37.8 & 12.1 & 11.4 & 10.5 & 3.11 & 3.32 & 3.60 \\
\texttt{fp32} & 3 & 512 & 44.1 & 12.5 & 11.5 & 11.0 & 3.54 & 3.85 & 4.02 \\
\texttt{fp32} & 4 & 512 & 52.9 & 14.8 & 13.5 & 13.9 & 3.56 & 3.91 & 3.79 \\
\bottomrule
\end{tabular}
\end{table*}

The quality story is equally encouraging. In every tested setting, the MSE of all IsoQuant variants is either indistinguishable from RotorQuant or slightly lower. For example, in the largest tested FP32 configuration $(d=512,b=4)$, all three IsoQuant variants remain at the same numerical scale as RotorQuant, and in many settings the printed MSE values are exactly identical. We therefore do not observe a tradeoff in which IsoQuant wins speed by degrading reconstruction quality.

\subsection{Result Interpretation}

The full sweep reveals three robust trends.
\begin{enumerate}[nosep,leftmargin=14pt]
  \item \textbf{IsoQuant-Fast and IsoQuant-$2$D are the lowest-latency operating points.} Their mean speedups are nearly identical, and each variant attains the best latency in a subset of the tested settings. We therefore view the $2$D case as a useful lightweight special case rather than a replacement for the $4$D family.
  \item \textbf{The gain is strong in both FP16 and FP32.} Averaged over all tested FP16 settings, IsoQuant-Full, IsoQuant-Fast, and IsoQuant-$2$D achieve $4.40\times$, $4.63\times$, and $4.53\times$ speedups, respectively; the corresponding FP32 averages are $4.57\times$, $4.69\times$, and $4.78\times$.
  \item \textbf{Low-bit and medium-width settings are especially favorable.} Several configurations in the $(d=128,256)$ range exceed $6\times$ speedup, showing that the compact blockwise formulations amortize per-block overhead particularly well in practical KV-cache regimes.
\end{enumerate}

These trends are consistent with the systems argument developed earlier. Compared with RotorQuant's $3$D Clifford blocks, IsoQuant avoids the expansion to an $8$-component multivector representation, keeps the per-block state smaller, and aligns naturally with two-wide or four-wide memory and register organization. The result is not a change in asymptotic complexity---all methods remain linear in $d$---but a real and repeatable reduction in constant factors.

\subsection{Module-Level vs Kernel-Level Measurements}

In addition to fused-kernel measurements, we also benchmarked higher-level PyTorch module paths. There, the apparent speedups can reach roughly $4\times$--$10\times$, because the IsoQuant prototype currently enjoys a more streamlined execution path than the baseline RotorQuant module. However, we regard the fused CUDA comparison as the fairest measure of method-intrinsic systems advantage. The fused results therefore serve as our primary hardware claim, while the larger module-level gains should be interpreted as an implementation-dependent systems outcome.

\subsection{Ablation Coverage}

\begin{table}[t]
\centering
\caption{Ablation axes. Completed settings in this work are marked by checkmarks.}
\label{tab:ablations}
\begin{tabular}{@{}ll@{}}
\toprule
Axis & Values \\
\midrule
Rotation type & Full, Fast, $2$D special case \hfill $\checkmark$ \\
Bit width & 2, 3, 4 bits \hfill $\checkmark$ \\
Block size & 3D, 4D, optionally 8D grouped variants \\
Quaternion parameters & random fixed \hfill $\checkmark$, learned normalized \\
Residual correction & off \hfill $\checkmark$, QJL-style correction on \\
Deployment backend & PyTorch \hfill $\checkmark$, fused CUDA \hfill $\checkmark$ \\
\bottomrule
\end{tabular}
\end{table}

Table~\ref{tab:ablations} summarizes the experimental axes covered by the current draft. The completed experiments already establish a strong stage-1 result across rotation type, bit width, and backend, but they do not yet include residual correction or downstream KV-cache task metrics.

\subsection{What Remains for a Full Submission}

The current experiments establish a strong stage-1 result, but a full submission should still add downstream KV-cache metrics:
\begin{enumerate}[nosep,leftmargin=14pt]
  \item integration with a stage-2 residual correction module such as QJL;
  \item attention-logit preservation and inner-product error under the complete two-stage pipeline;
  \item retrieval and perplexity experiments on real KV tensors extracted from deployed LLMs.
\end{enumerate}
These follow naturally from the present implementation because IsoQuant only replaces the stage-1 transform and remains compatible with the same residual correction machinery used by TurboQuant and RotorQuant.

\section{Limitations and Future Work}
\label{sec:limitations}

IsoQuant is not a full solution by itself.
\begin{enumerate}[nosep,leftmargin=14pt]
  \item \textbf{Block locality.} Like other block-diagonal transforms, it does not mix information across blocks, so global correlations remain unaddressed.
  \item \textbf{Evaluation gap.} This draft now includes fused CUDA benchmarks and reconstruction experiments, but the final conference version still needs end-to-end KV-cache validation on real models and tasks.
  \item \textbf{Learning dynamics.} Although normalized quaternion parameters are simple to optimize, the relative value of learned versus random rotations remains an empirical question.
  \item \textbf{Stage-2 interaction.} The best coupling between $4$D block decorrelation and residual correction may depend on the downstream attention estimator.
\end{enumerate}

Promising next steps include hierarchical cross-block mixing, jointly learned codebooks and quaternion parameters, and specialized kernels for fused KV-cache compression during autoregressive decoding.

\section{Conclusion}
\label{sec:conclusion}

IsoQuant replaces awkward $3$D block rotations with a hardware-aligned $4$D quaternion formulation grounded in the isoclinic decomposition of $SO(4)$. The resulting design retains the spirit of blockwise decorrelation while offering stronger local mixing, cleaner memory alignment, and lower arithmetic cost. Our fused CUDA experiments on practical dimensions $d \in \{128,256,512\}$ show that this design advantage is measurable in practice: across $18$ matched settings, IsoQuant-Full, IsoQuant-Fast, and the lightweight $2$D special case deliver average speedups of $4.49\times$, $4.66\times$, and $4.66\times$ over RotorQuant fused CUDA while preserving essentially identical reconstruction MSE. We therefore view the $2$D case as a useful low-cost operating point, while the full $4$D construction remains the primary and most expressive form of the method. These results make a strong case that blockwise isoclinic rotations are a compelling stage-1 replacement for online vector quantization, and they motivate the next step of integrating the method with residual correction and full KV-cache evaluation. An additional advantage of quaternion-pair parameterization is that it admits smooth interpolation on the underlying rotation manifold, which may be useful for shared or adaptive block rotations in future work.

\bibliographystyle{plainnat}

\end{document}